\documentclass{article}

\usepackage{microtype}
\usepackage{graphicx}
\usepackage{subcaption}
\usepackage{booktabs} 
\usepackage{kotex}

\usepackage{hyperref}

\usepackage{natbib}
\usepackage[preprint]{icml2026}

\usepackage{amsmath}
\usepackage{amssymb}
\usepackage{mathtools}
\usepackage{amsthm}
\usepackage{multirow} 

\usepackage[capitalize,noabbrev]{cleveref}

\theoremstyle{plain}
\newtheorem{theorem}{Theorem}[section]
\newtheorem{proposition}[theorem]{Proposition}
\newtheorem{lemma}[theorem]{Lemma}
\newtheorem{corollary}[theorem]{Corollary}
\theoremstyle{definition}

\theoremstyle{remark}

\usepackage[textsize=tiny]{todonotes}

\icmltitlerunning{MATE: Memory of Accumulated Transition Embeddings}

\begin{document}

\twocolumn[
  \icmltitle{MATE: Solving Contextual Markov Decision Processes with \\Memory of Accumulated Transition Embeddings}



\icmlsetsymbol{equal}{*}
  \begin{icmlauthorlist}
    \icmlauthor{Himchan Hwang}{snu}
    \icmlauthor{Hyeokju Jeong}{snu}
    \icmlauthor{Gene Chung}{sage}
    \icmlauthor{Seungyeon Kim}{snu}\\
    \icmlauthor{Sangwoong Yoon}{unist}
    \icmlauthor{Frank Chongwoo Park}{snu,sage}
  \end{icmlauthorlist}
  \icmlaffiliation{snu}{Seoul National University}
  \icmlaffiliation{sage}{Saige}
  \icmlaffiliation{unist}{Ulsan National Institute of Science and Technology (UNIST)}

  \icmlcorrespondingauthor{Sangwoong Yoon}{swyoon@unist.ac.kr}
  \icmlcorrespondingauthor{Frank Chongwoo Park}{ fcp@snu.ac.kr}

  \icmlkeywords{Machine Learning, ICML}

  \vskip 0.3in
]



\printAffiliationsAndNotice{}  

\begin{abstract}
    We propose MATE, a simple yet effective memory architecture for solving Contextual Markov Decision Processes (CMDPs), a family of MDPs parameterized by an unobserved context. In CMDPs, an optimal agent can adapt online by maintaining the posterior belief over contexts. MATE replaces this intractable posterior with a sum-aggregated memory, leveraging the posterior’s permutation invariance to retain provably sufficient expressiveness. Compared to prior memory architectures, MATE avoids the growing per-step rollout cost of Transformers and the gradient issues commonly associated with Recurrent Neural Networks (RNNs). Extensive evaluations across diverse benchmarks demonstrate that MATE provides clear computational advantages while achieving performance comparable to standard sequence-model baselines.
\end{abstract}

\section{Introduction}

Solving Contextual Markov Decision Processes (CMDPs) \cite{hallak2015contextual} enables RL agents to operate beyond fixed simulation environments, extending to real-world applications with inherent variations, such as sim-to-real transfer, healthcare, and autonomous driving \cite{kirk2023survey}. In these settings, environmental variations are governed by a latent context. To adapt effectively, an optimal agent must infer this context from its interaction history—a capability beyond the reach of standard state-dependent policies.

A prominent solution is memory-based RL, where sequence models are trained end-to-end to encode transition histories into a policy-accessible memory. \citet{ni2022recurrent} demonstrated the efficacy of this approach using RNNs, outperforming prior specialized methods. More recently, \citet{grigsby2024amago,grigsby2024amago2} achieved strong performance in multi-task and long-horizon domains by leveraging Transformer architectures. However, these architectures face distinct limitations: RNNs suffer from gradient instabilities \cite{bengio1994learning, pascanu2013difficulty} and preclude temporal parallelization during the update step \cite{lu2023structured}. Conversely, Transformers incur quadratic computational costs with respect to sequence length during both the update step and trajectory rollout, often necessitating a limited context window \cite{beck2025tutorial}.

In this paper, we propose a memory-based RL agent that exploits the {\it permutation invariance} of the context posterior in CMDPs (see Figure~\ref{fig:illustration}). In a CMDP, the posterior belief over the context is invariant to the ordering of transition data. We argue that an optimal memory-based RL agent should inherently incorporate this inductive bias. While bidirectional Transformers can capture such invariance, the cost of re-encoding the full history at each step is prohibitive. Thus, causal Transformers are generally utilized in memory-based RL despite lacking this property.

\begin{figure*}[!t]
  \vskip 0.2in
  \begin{center}
    \centerline{\includegraphics[width=2.0\columnwidth]{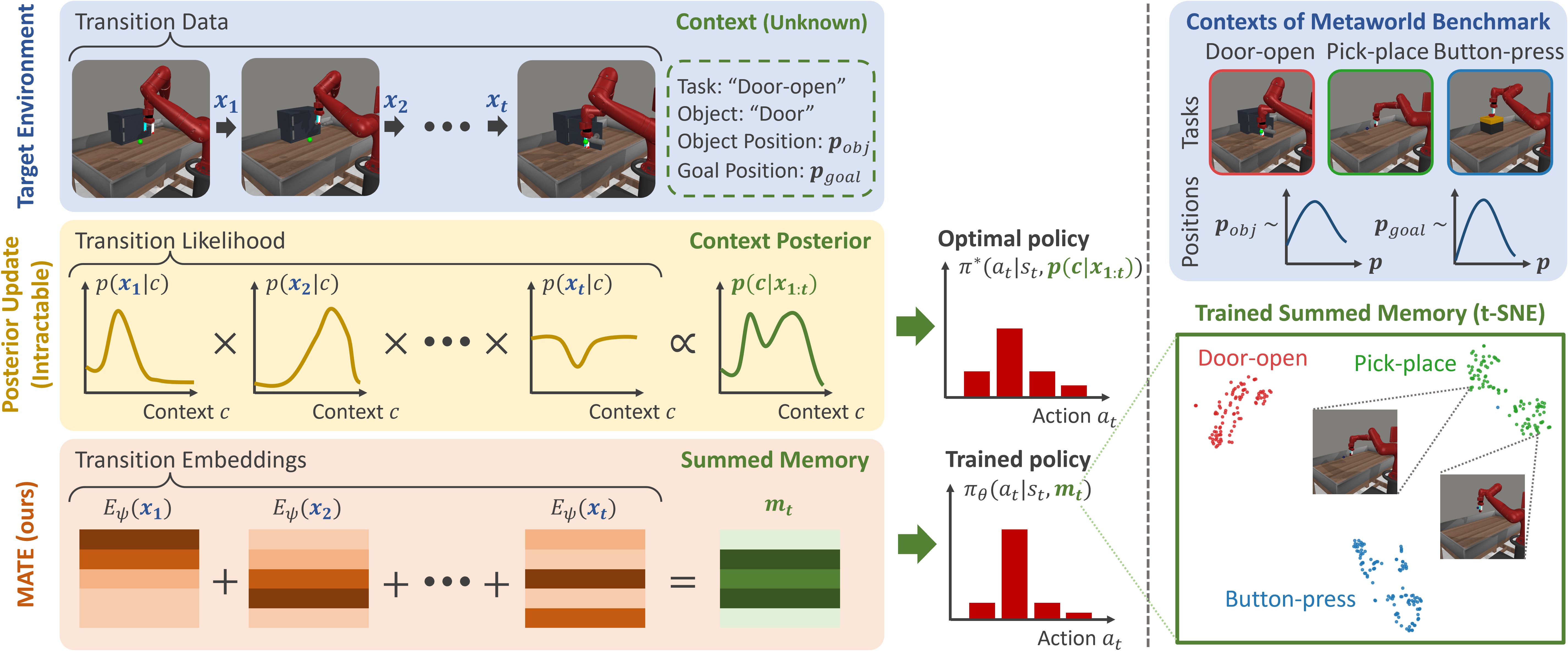}}
        \caption{
    \textbf{Overview of MATE.} MATE represents the memory $m_t$ as a summation of transition embeddings, serving as a tractable substitute for the intractable posterior $p(c|x_{1:t})$. By preserving the permutation invariance of the posterior, $\pi_{\theta}$ is provably capable of representing the optimal policy $\pi^*$, despite its structural simplicity.
    }
    \label{fig:illustration}
  \end{center}
\end{figure*}

Motivated by this perspective, we propose \textbf{MATE} (\textbf{M}emory of \textbf{A}ccumulated \textbf{T}ransition \textbf{E}mbeddings), a sum-aggregated memory architecture for solving CMDPs (see lower-left of Figure~\ref{fig:illustration}). By design, MATE guarantees permutation invariance across transition histories while avoiding the complexity of recurrent layers and attention mechanisms. Notably, MATE preserves the computational advantages of prior approaches: it enables temporal parallelization during the update step (like Transformers) while maintaining a linear computational cost with respect to sequence length (like RNNs). While \citet{beck2024splagger} employed a similar mean aggregation of linear embeddings as a baseline, we extend this idea by adopting sum aggregation over embeddings produced by arbitrary neural networks.

We validate our model from both theoretical and empirical perspectives. Theoretically, we show that an agent equipped with MATE retains sufficient expressiveness to represent the optimal policy by leveraging permutation invariance in CMDPs. Empirically, we demonstrate that MATE achieves competitive or superior performance compared to Transformer and RNN baselines across diverse benchmarks. Our results highlight its effective adaptation to variations in dynamics, rewards, and tasks (MuJoCo, Meta-World), as well as its handling of long-term memorization and credit assignment challenges (T-Maze).

\section{Preliminaries}
\subsection{Contextual MDP}\label{sec:cmdp}
A Contextual Markov Decision Process (CMDP) \citep{hallak2015contextual} is defined as a set of MDPs $\{(\mathcal{S}, \mathcal{A}, P_c) \mid c\in\mathcal{C}\}$, where $\mathcal{S}$, $\mathcal{A}$, and $\mathcal{C}$ denote the state, action, and context spaces. The \emph{unobserved} context $c$ determines the transition kernel $P_c$, given by $p(s_{t+1}, r_t |s_t, a_t, c)$. A CMDP can be viewed as a special case of a partially observable MDP (POMDP) \citep{cassandra1994acting} where the unobserved state is the context.

The goal in a CMDP is to find an optimal policy $\pi^*$ that maximizes the expected discounted return under the context prior $p(c)$:
\begin{align}
\max_{\pi} \ \mathbb{E}_{c \sim p(c), \pi}\left[\sum_{t=0}^{\infty}\gamma^t r_t \right],
\label{eq:objective}
\end{align}
where $\gamma$ is the discount factor. Thus, the optimal agent generalizes across MDPs induced by different contexts and adapts to the current MDP within an episode.

For CMDPs, the optimal policy and value function can be expressed as functions of the current state $s_t$ and the posterior belief over contexts $p(c | x_{1:t})$: 
\begin{align}
    \pi^*(a_t | s_t, p(c | x_{1:t})), \quad V^*(s_t,p(c | x_{1:t})). \label{eq:opt_policy}
\end{align}
Here, $x_t = (s_{t-1}, a_{t-1}, r_{t-1}, s_t)$ denotes the transition at time $t$. This formulation is derived by casting the CMDP as an MDP with an augmented state $s_t^+ = (s_t, p(c | x_{1:t}))$ \cite{duff2002optimal, zintgraf2019varibad}. However, exact posterior inference is typically intractable, as it requires knowledge of the unknown transition kernel $P_c$ and may not admit a finite-dimensional representation.

\subsection{Memory-based RL} 
To learn history-dependent policies as required in CMDPs (\cref{eq:opt_policy}), leveraging sequence models to compress the transition history into a latent memory has emerged as a promising approach \citep{ni2022recurrent, beck2023recurrent, grigsby2024amago, grigsby2024amago2}. In this work, we refer to this general framework as \textbf{memory-based RL}.

Memory-based RL employs a memory network $m_{\psi}$, parametrized by $\psi$, that summarizes the transition history into a finite-dimensional memory $m_t = m_{\psi}(x_{1:t}) \in \mathbb{R}^m$. This is then used to condition the memory-augmented policy and value networks:
\begin{align}
    \pi_{\theta}(a_t | s_t, m_t), \quad V_{\phi}(s_t, m_t). \label{eq:memory_arch}
\end{align}
All networks are jointly trained with RL. Existing methods vary in the critic formulation ($V_{\phi}$, $Q_{\phi}$, or none) and in how the networks are conditioned (on $(s_t,m_t)$ or on $m_t$ alone). We use \cref{eq:memory_arch} as a representative case in our analysis. 

Transformers \cite{grigsby2024amago, grigsby2024amago2,shala2025efficient} and RNNs \cite{duan2016rl,ni2022recurrent,beck2023recurrent}, or their variants \cite{beck2020amrl,beck2024splagger, huang2024decision}, are the standard choices for $m_{\psi}$. RNNs offer the advantage of constant inference costs during rollout, but their recurrent nature often leads to gradient instabilities during the update step \cite{bengio1994learning, pascanu2013difficulty}. Transformers excel at long-term memorization \cite{ni2023transformers}, but their reliance on attention mechanisms incurs linearly growing inference costs per rollout step.

\section{Methods}
In this section, we introduce \textbf{M}emory of \textbf{A}ccumulated \textbf{T}ransition \textbf{E}mbeddings (\textbf{MATE}), a simple sum-aggregated memory architecture for CMDPs that addresses the limitations of Transformers and RNNs.
We first formalize permutation invariance in CMDPs in \cref{sec:perm_inv}, which serves as the foundation of our model. We then describe the detailed architecture of MATE and its computational advantages over Transformers and RNNs in \cref{sec:mate}. In \cref{sec:theoretical_analysis}, we establish that MATE provides a theoretically grounded representation for solving CMDPs. Finally, we introduce a normalization technique in \cref{sec:hypersphere} to prevent memory divergence in practice.

\subsection{Permutation Invariance in CMDPs} \label{sec:perm_inv} 
The permutation invariance of the context posterior $p(c|x_{1:t})$ is a fundamental property in CMDPs that serves as both the primary motivation and the theoretical foundation for our architecture. We therefore explicitly formalize this property in \cref{eq:perm_inv}, noting its utilization in prior studies \citep{imagawa2022off, beck2024splagger}. By applying Bayes' rule and leveraging the conditional independence of transitions given $c$, the posterior factorizes as:
\begin{align}
p(c | x_{1:t}) \propto p(c)\prod_{i=0}^{t-1} p(s_{i+1}, r_{i} |s_{i}, a_{i}, c),
\label{eq:context_posterior}
\end{align}
assuming that $s_0$ is independent of $c$ (see \cref{appendix:factorization} for full derivation). Since the product operation in \cref{eq:context_posterior} is commutative, the posterior is invariant to the order of transitions. Thus, for any permutation $\rho:\{1,\ldots,t\}\to\{1,\ldots,t\}$,
\begin{align}
p(c | x_1, \ldots, x_t) = p(c | x_{\rho(1)}, \ldots, x_{\rho(t)}).\label{eq:perm_inv}
\end{align}

\begin{figure}[!t]
  \vskip 0.2in
  \begin{center}
    \centerline{\includegraphics[width=0.8\columnwidth]{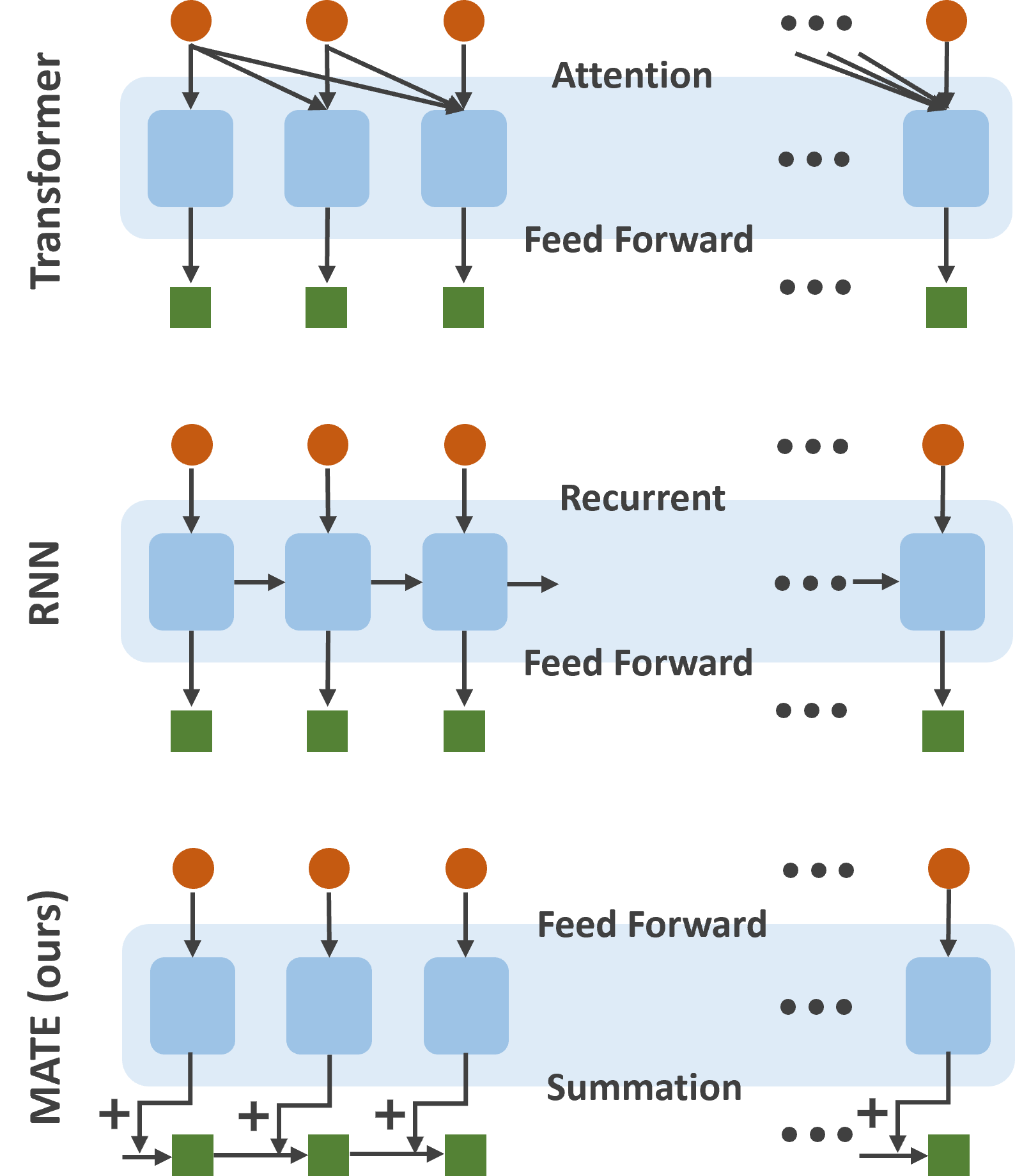}}
    \caption{
    \textbf{Comparison of memory architectures.} MATE replaces the attention mechanisms of Transformers and the recurrent connections of RNNs with simple sum aggregation. Single-layer versions are illustrated for clarity.
    }
    \label{fig:model_comparison}
  \end{center}
\end{figure}

\subsection{Memory of Accumulated Transition Embeddings}
\label{sec:mate}
A desirable memory $m_t$ for CMDPs should act as a tractable surrogate for the context posterior $p(c | x_{1:t})$, as suggested by \crefrange{eq:opt_policy}{eq:memory_arch}. Motivated by the permutation invariance of $p(c | x_{1:t})$, we introduce \textbf{M}emory of \textbf{A}ccumulated \textbf{T}ransition \textbf{E}mbeddings (\textbf{MATE}), a permutation-invariant memory defined as
\begin{align}
    m_t = \sum_{i=1}^t E_{\psi}(x_i). \label{eq:mate}
\end{align}
The transition encoder $E_{\psi}$ can be any network that maps each transition $x_i$ to an embedding. As highlighted in \cref{fig:model_comparison}, MATE summarizes the transition history through simple accumulation, unlike Transformer or RNN-based memories that utilize attention or recurrent layers.

To provide intuition for why summation can be a natural choice, consider a case where the context posterior takes the form of a product of Gaussian factors, $p(c | x_{1:t}) \propto \prod_{i=1}^t \mathcal{N}(c; \mu(x_i), \sigma^2(x_i))$, as in \citet{rakelly2019efficient}. If we set $E_{\psi}(x_i) = (\mu(x_i)/\sigma^2(x_i), 1/\sigma^2(x_i))$, then the accumulated memory $m_t$ exactly specifies $p(c | x_{1:t})$ (see \cref{appendix:permutation} for details).

MATE offers computational efficiency compared to other sequence models, as summarized in \cref{tab:model-comparison}. Specifically, MATE maintains a constant inference cost per rollout step via the update rule: $m_{t+1}=m_t+E_{\psi}(x_{t+1})$. In contrast, causal Transformers incur an $\mathcal{O}(t)$ cost per step due to the growing attention window. MATE also supports parallel update across the sequence dimension, a capability lacking in RNNs \citep{lu2023structured}. We note that while bidirectional Transformers could theoretically offer permutation invariance, they are impractical for autoregressive rollout as they typically require reprocessing the full history at each step. To the best of our knowledge, prior memory-based RL work has exclusively adopted causal Transformers.

\subsection{Theoretical Optimality of MATE}\label{sec:theoretical_analysis} MATE cannot represent arbitrary history-dependent policies, as it constrains the policy to be independent of the transition order. However, we demonstrate in this section that MATE remains a sufficiently expressive memory architecture for CMDPs by leveraging the permutation invariance of the context posterior established in \cref{sec:perm_inv}.

We first establish the necessary notation and assumptions. Recall that $x_i = (s_{i-1}, a_{i-1}, r_{i-1}, s_i)$ denotes a transition sample. We assume the space of transitions $\mathcal{X} \subset \mathbb{R}^n$ is compact, as states and actions are typically bounded real vectors in practice. Given the episode length $T$, let $\mathcal{M}_{\leq T}(\mathcal{X})$ denote the collection of all multisets $\{\{x_1,\ldots,x_t\}\}$ with elements in $\mathcal{X}$ and $t\leq T$. Crucially, the permutation invariance of $p(c|x_{1:t})$ (\cref{eq:perm_inv}) implies that the optimal policy defined in \cref{eq:opt_policy} can be viewed as a mapping $\pi^*: \mathcal{S} \times \mathcal{M}_{\leq T}(\mathcal{X})\to\mathcal{Y}$. While $\mathcal{Y}$ theoretically represents a space of probability measures, practical implementations parameterize these distributions using finite-dimensional vectors (e.g., probability simplex for discrete actions, Gaussian parameters for continuous actions). Thus, without loss of generality, we assume $\mathcal{Y} \subset \mathbb{R}^k$.

In the following proposition, we show that with sufficient memory dimension, MATE with only a one-layer MLP transition encoder suffices to represent the optimal policy.

\begin{proposition}
\label{proposition:optimality}
Let $E_{\psi}:\mathcal{X}\to\mathbb{R}^m$ be a one-layer MLP with output dimension $m=2nT+1$ and an analytic non-polynomial element-wise activation $f:\mathbb{R}\to\mathbb{R}$, defined as
\begin{align}
    E_{\psi}(x) = f(Ax + b), \ \text{where } \psi=(A,b).
\end{align}

Then, given a continuous optimal policy $\pi^*:\mathcal{S}\times\mathcal{M}_{\leq T}(\mathcal{X})\to\mathcal{Y}\subset\mathbb{R}^k$, there exists a continuous function $\pi:\mathcal{S}\times\mathbb{R}^m\to\mathbb{R}^k$ such that
\begin{align}
    \pi(s_t, \sum_{i=1}^t E_{\psi}(x_i)) = \pi^*(s_t, p(c|x_{1:t}))
\end{align}
for Lebesgue-almost-every $A\in\mathbb{R}^{m\times n}$ and $b\in\mathbb{R}^m$.
\end{proposition}

\noindent\textit{Proof sketch.}
We construct $\pi$ as the composition of mappings:
\begin{align}
(s_t, m_t) \mapsto \ (s_t, \{\{x_1, \dots, x_t\}\}) \mapsto  \pi^*(s_t, p(c| x_{1:t})).
\end{align}
The first map is well-defined because the summation $m_t = \sum E_\psi(x_i)$ is injective \citep[Theorem~3.3]{amir2023neural}, and the second is the optimal policy itself. The continuity is guaranteed by the Tietze extension theorem. See \cref{appendix:prop_1} for details.

\begin{table}[!t]
\caption{
Total computational complexity of \textbf{rollout} and \textbf{update} over the episode length $T$, and whether the update is \textbf{parallelizable} across the sequence dimension, comparing Transformers (causal), RNNs, and MATE.
}
\label{tab:model-comparison}
\begin{center}
\setlength{\tabcolsep}{8pt}
\begin{tabular}{lccc}
\toprule
\textbf{Model} & \textbf{Rollout} & \textbf{Update} & \textbf{Parallelizable} \\
\midrule
Transformers & $\mathcal{O}(T^2)$ & $\mathcal{O}(T^2)$ &\textbf{Yes} \\
RNNs                  & $\mathcal{O}(\boldsymbol{T})$ & $\mathcal{O}(\boldsymbol{T})$ &  No  \\
\textbf{MATE (Ours)}  & $\mathcal{O}(\boldsymbol{T})$ & $\mathcal{O}(\boldsymbol{T})$ & \textbf{Yes} \\
\bottomrule
\end{tabular}
\end{center}
\vskip -0.1in
\end{table}

This existence result directly extends to the optimal value function with the same choice of $E_{\psi}$,  which justifies the architectural choice of sharing the memory $m_t$ between the policy and value networks.

\begin{corollary}\label{corollary:value_function}
Under the same assumptions as \cref{proposition:optimality}, for a continuous optimal value function $V^*:\mathcal{S}\times\mathcal{M}_{\leq T}(\mathcal{X})\to\mathbb{R}$, there exists a continuous function $V:\mathcal{S}\times\mathbb{R}^m\to\mathbb{R}$ such that
\begin{align}
    V(s_t, \sum_{i=1}^t E_{\psi}(x_i)) = V^*(s_t, p(c|x_{1:t}))
\end{align}
\end{corollary}

In practice, we use neural networks to approximate $\pi$ and $V$ (\cref{eq:memory_arch}). Given their continuity established in \cref{proposition:optimality} and \cref{corollary:value_function}, the universal approximation theorem \citep{cybenko1989approximation} implies that the policy and value networks equipped with MATE can approximate the optimal policy and value function to arbitrary accuracy.

\begin{figure*}[!t]
  \vskip 0.2in
  \begin{center}
    \centerline{\includegraphics[width=2.0\columnwidth]{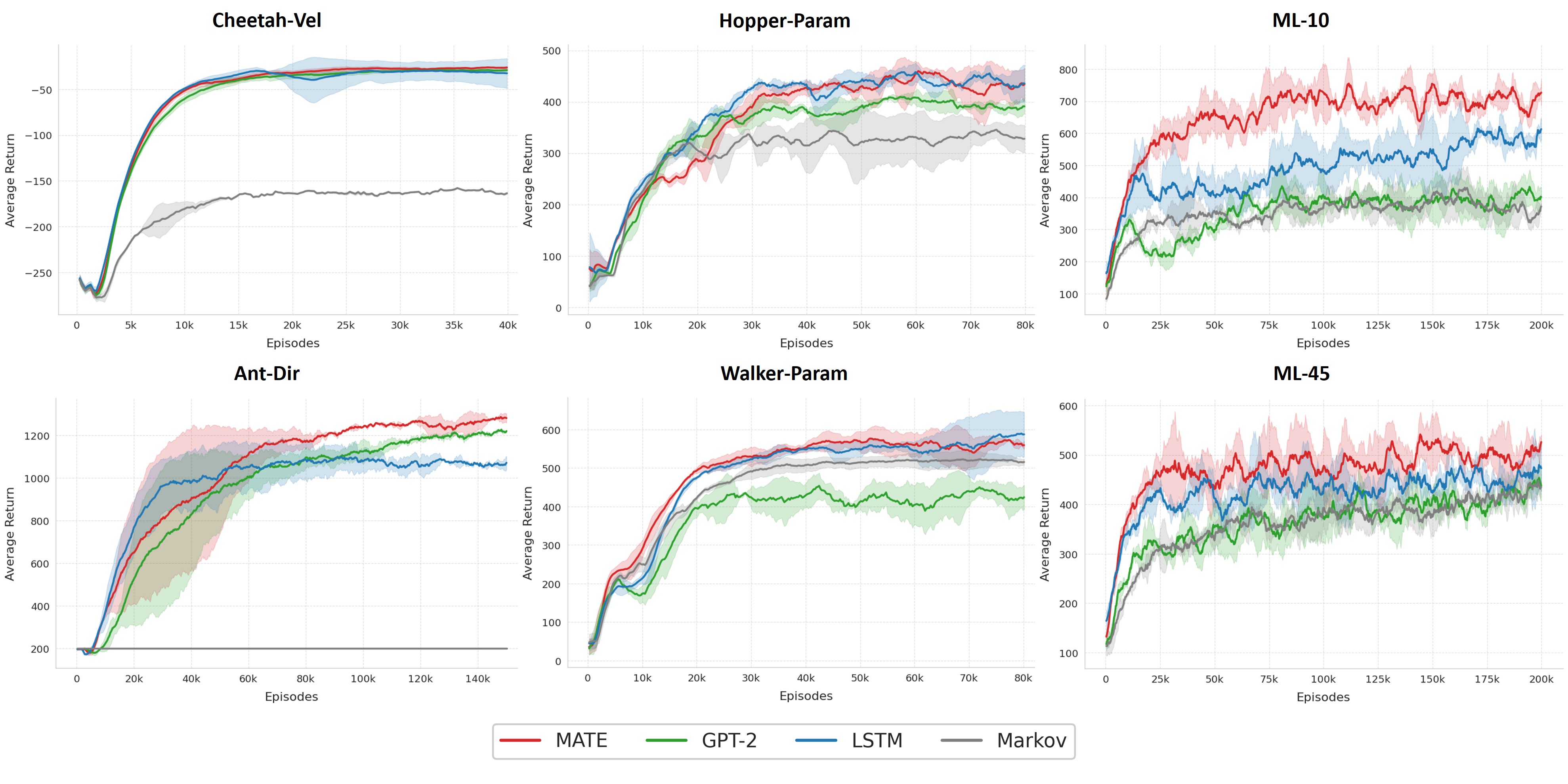}}
    \caption{
    \textbf{Learning curves on MuJoCo (Left, Center) and Meta-World (Right) benchmarks.} Tasks are ordered by increasing difficulty from top to bottom. Solid lines and shaded regions represent the mean and standard deviation across 3 random seeds. All agents are trained using Soft Actor-Critic (SAC)~\cite{haarnoja2018soft}
    }
    \label{fig:results_mujoco_and_metaworld}
  \end{center}
\end{figure*}

\subsection{Hyperspherical Memory Normalization}\label{sec:hypersphere}
In memory-based RL, the memory $m_t$ serves as an input to the policy and value networks. Here, the sum aggregation of MATE can be problematic, as the norm of $m_t$ typically grows with $t$. This unbounded growth poses challenges for neural network training, as it induces internal covariate shift and destabilizes optimization \citep{ioffe2015batch}. To address this issue, we normalize the memory by projecting it onto a hypersphere:
\begin{align}
\hat{m}_t=\frac{m_t+\Psi}{\left\lVert m_t+\Psi \right\rVert}.
\end{align}
We demonstrate that the guarantee in \cref{proposition:optimality} remains valid with a memory dimension of $m=2nT+2$ (see \cref{appendix:extension}). The offset $\Psi$ is a necessary component that prevents information loss during normalization; without it, distinct histories such as $\{\{x,x\}\}$ and $\{\{x\}\}$ would yield the same normalized memory. 

In our experiments, we treat $\Psi$ as a trainable parameter. We also scale $\hat{m}_t$ by the square root of its dimension, adopting the strategy in RMSNorm \citep{zhang2019root}.

\section{Experiments}
We evaluate MATE against Transformer and RNN baselines on three distinct CMDP benchmarks: MuJoCo, Meta-World, and T-Maze, selected to provide a comprehensive assessment of memory-based RL agents. We detail the baselines and experimental setup in \cref{sec:baseline}, followed by main results in \cref{sec:results}. Additional details are provided in \cref{appendix:exp_details}.

\subsection{Baselines and Experiment Settings}
\label{sec:baseline}
To evaluate the performance of MATE as a memory architecture, we compare it against standard RNN and Transformer baselines, specifically LSTM \cite{hochreiter1997long} and GPT-2 \cite{radford2019language}, following the \citet{ni2023transformers}. Additionally, we include a memory-free Markovian policy as a baseline for ablation. For this baseline, we report the optimal policy where available; otherwise, we train a memory-less agent from scratch.

We align the network architectures across all baselines to ensure a fair comparison. All models share identical policy and value network structures. Since GPT-2 requires an embedding layer for residual connections, we incorporate an embedding layer into all baselines for consistency. To maintain comparable model complexity, we configure both the LSTM and GPT-2 baselines with a single layer. Similarly, MATE utilizes a single feed-forward block adapted from the GPT-2 architecture. The memory dimension is unified across all models. Further implementation details are provided in \cref{appendix:baselines}.

\subsection{Benchmarks}
\label{sec:baseline_benchmark}
\textbf{MuJoCo.} \quad 
The MuJoCo benchmark is a widely adopted suite in Meta-RL literature \citep[e.g.,][]{finn2017maml, rakelly2019efficient, li2024towards}, comprising four variants of standard MuJoCo locomotion environments \citep{todorov2012mujoco}. These environments are specifically designed to assess an agent’s adaptability to variations in either the \emph{reward function} (Ant-Dir, Cheetah-Vel) or \emph{system dynamics} (Walker-Param, Hopper-Param). Among them, Ant-Dir and Walker-Param are considered more challenging benchmarks due to their higher degrees of freedom. We set the episode length to $T=200$ for all tasks. Detailed configurations are provided in \cref{appendix:mujoco}.

\textbf{Meta-World.} \quad 
Meta-World benchmark \citep{yu2020meta} features diverse robot manipulation tasks with varying objects and goals that share a unified state and action space. In our experiments, we utilize the ML10 benchmark, comprising 10 distinct tasks (e.g., \emph{door-open}, \emph{pick-place}). Although ML10 includes out-of-distribution test tasks, we focus on evaluating adaptability to the variation within the training tasks. Since task descriptions are not provided, the agent must infer the underlying task solely from the interaction history. The episode length is set to $T=300$. Detailed configurations are provided in \cref{appendix:metaworld}.

\textbf{T-Maze.} \quad
T-Maze \cite{ni2023transformers} is a 2-dimensional navigation task where an agent must acquire goal information from an oracle state to successfully reach the target. We employ two variants, \textit{Passive} and \textit{Active}, which differ in the oracle's configuration to assess memory retention and credit assignment capabilities, respectively. While the original T-Maze is formulated as a POMDP, we reformulate its state representation and reward structure to align with the CMDP framework. We evaluate MATE and baselines across corridor lengths ranging from 200 to 600, with a fixed budget of 80,000 training episodes. Detailed configurations are provided in \cref{appendix:tmaze_setting}.


\begin{figure}[t]
  \vskip 0.2in
  \begin{center}
    \centerline{\includegraphics[width=0.95\columnwidth]{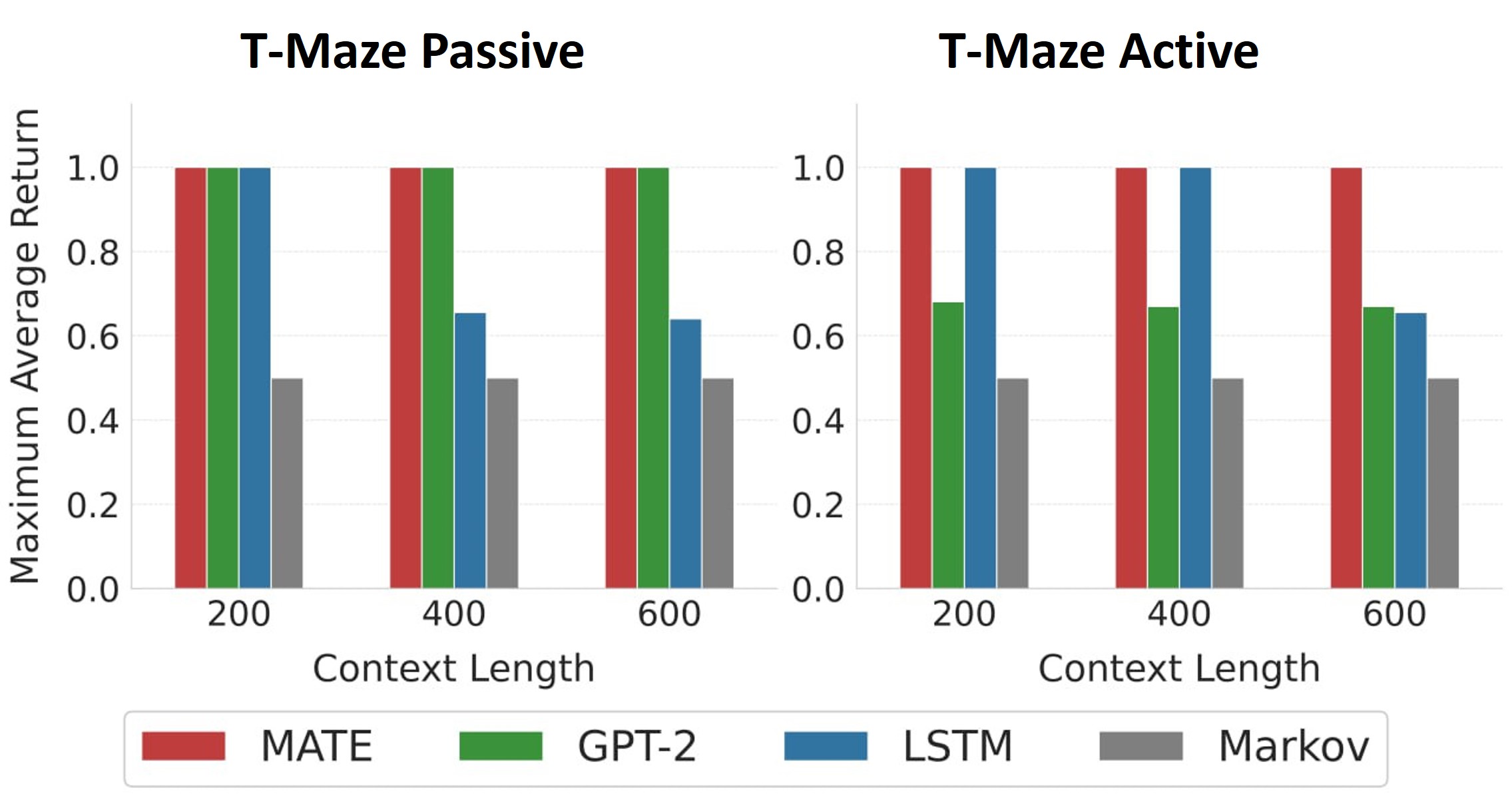}}
    \caption{
    \textbf{Performance comparison on the T-Maze tasks.} The figures show the maximum average return on  the (Left) Passive T-Maze and (Right) Active T-Maze environments. The reported values are the best results from 2 independent runs. All agents are trained using Double Q-Learning (DDQN) \cite{van2016deep}.
    }
    \label{fig:results_tmaze}
  \end{center}
\end{figure} 

\subsection{Results}
\label{sec:results}
\paragraph{MuJoCo\& Meta-World}\quad 
Figure~\ref{fig:results_mujoco_and_metaworld} shows learning curves on MuJoCo (left) and Meta-World (right), with tasks ordered by increasing difficulty. On MuJoCo, MATE performs on par with a diverse set of memory and sequence-based baselines (e.g., recurrent and
Transformer-style variants), achieving similar learning speed and final performance. On Meta-World, MATE consistently outperforms these baselines, attaining higher returns and improved sample efficiency, despite its lightweight architecture. Notably, MATE can be viewed as a simplified Transformer that removes positional encoding and attention and instead uses permutation-invariant summation, suggesting that sum-based memory updates are particularly effective for learning from transition data.

\paragraph{T-Maze} \quad
Experimental results for the T-Maze tasks are presented in \cref{fig:results_tmaze}. In the Passive T-Maze, both MATE and GPT-2 achieve optimal returns across all context lengths, whereas LSTM shows a decrease in performance as the sequence length grows. In the Active T-Maze, MATE maintains its performance, while GPT-2 shows lower returns compared to the passive task. LSTM shows comparable performance with MATE up to a context length of 400 but exhibits a decline at the length of 600. Overall, MATE demonstrates the most consistent performance across both tasks and varying context lengths.



\section{Related Work}
\label{sec:related}

\subsection{Meta-Reinforcement Learning}

Meta-RL studies how agents can adapt across a distribution of tasks from interaction history \cite{beck2025tutorial}. A common distinction is between \emph{task-inference} methods, which explicitly estimate a latent task or context variable, and \emph{black-box} methods, which train a history-dependent policy end-to-end \citep{duan2016rl, rakelly2019efficient, zintgraf2019varibad, ni2022recurrent, beck2023recurrent}. Task-inference methods such as PEARL, VariBAD, and ELUE introduce dedicated objectives for context estimation, often decoupling inference from control \citep{rakelly2019efficient, zintgraf2019varibad, imagawa2022off, li2024towards}. This separation can improve training stability and interpretability, but it typically requires auxiliary losses, additional decoders, or parametric assumptions on the posterior.

Black-box methods instead compress interaction history directly into memory and optimize the entire agent through the RL objective. From RL$^2$ to more recent recurrent and Transformer agents, this line has become a strong baseline for CMDPs and related POMDP settings \citep{duan2016rl, ni2022recurrent, beck2023recurrent, grigsby2024amago, grigsby2024amago2}. Its main limitation, however, is that the memory is usually implemented with order-sensitive sequence models, even though the CMDP context posterior is permutation-invariant with respect to transitions. MATE follows the black-box formulation, but incorporates this inference-side prior directly into the architecture: it retains end-to-end RL training while replacing order-sensitive memory with a permutation-invariant representation of history.

\subsection{Memory Architectures for RL}
\label{sec:related-memory}

Memory in RL is typically implemented with recurrent networks or causal attention. Recurrent agents provide constant per-step inference cost, but can be difficult to optimize over long horizons and do not support temporal parallelization during training \citep{duan2016rl, ni2022recurrent, beck2023recurrent, bengio1994learning, pascanu2013difficulty, lu2023structured}. Transformer-based agents substantially improve long-range memory, but their cost grows with context length during both training and rollout \citep{ni2023transformers, grigsby2024amago, grigsby2024amago2, beck2025tutorial}. More recent alternatives, including structured state-space models, Mamba-style architectures, and efficient Transformer variants, aim to improve this trade-off while remaining order-sensitive sequence models \citep{lu2023structured, gu2024mamba, beaussant2025scaling, elawady2024relic, shala2025efficient}.

MATE is motivated by a complementary perspective: in CMDPs, the key statistic for adaptation is permutation-invariant, so this invariance can be built into the memory itself. In this sense, MATE reduces complexity not by designing a better order-sensitive sequence model, but by removing the need for one when the problem structure allows it.

\subsection{Permutation-Invariant Aggregation}
\label{sec:related-permutation}

The most closely related works are those that summarize history with permutation-invariant aggregation. In black-box RL, AMRL and especially SplAgger are the closest precedents: they show that invariant aggregation is useful in end-to-end meta-RL, but retain a permutation-variant recurrent branch \citep{beck2020amrl, beck2024splagger}. In task-inference methods, PEARL and ELUE also exploit permutation-invariant context encoders; among them, ELUE is particularly close in its use of Deep-Sets-style sum aggregation, though within an explicit latent-inference pipeline trained with auxiliary objectives \citep{rakelly2019efficient, imagawa2022off}.

A related idea appears in IIDA, which aggregates unordered transition tuples for dynamics adaptation through an auxiliary prediction objective rather than end-to-end RL \citep{evans2022context}.

MATE differs from these lines in a simple but important way: permutation-invariant aggregation is the agent's \emph{only} memory mechanism. There is no recurrent backbone, no attention module, and no separate inference objective. This design is supported by the Deep Sets perspective \citep{zaheer2017deep} and by our sufficiency analysis, which shows that sum-aggregated transition embeddings can retain enough information to represent the optimal CMDP policy. Compared with prior invariant architectures, MATE thus offers a cleaner black-box formulation tailored to the structure of CMDPs.

\section{Conclusion}
We propose MATE, a simple yet effective memory architecture for agents solving CMDPs. By adopting a sum-aggregated architecture, we design a permutation-invariant memory model that effectively infers appropriate contextual representations of unknown environments from transition histories. Unlike prior approaches such as RNNs and Transformers, our model offers computational advantages during rollout and training, thereby improving training efficiency. Beyond its efficiency, we theoretically validate that our model does not compromise the expressiveness of the agent, in the sense that a globally optimal policy can be represented within our formulation.
We demonstrate that MATE achieves promising performance on several CMDP benchmarks, including MuJoCo, Meta-World and T-Maze, suggesting that our model effectively identifies relevant contextual information during rollout.

A promising avenue for future research lies in integrating our architecture with deep, self-supervised RL frameworks. Recent studies indicate that combining deep networks with self-supervised losses can yield substantial performance gains \citep{wang20251000}. Extending MATE to such high-capacity models could further enhance its scalability and robustness in complex environments.


\section*{Impact Statement}
This paper advances the robustness and adaptability of AI systems by proposing an effective model for reinforcement learning in the presence of unobserved contextual information, a common setting in real-world applications. We do not anticipate any direct negative impacts arising from this work. \looseness=-1

\bibliography{references}
\bibliographystyle{icml2026}

\newpage
\appendix
\onecolumn
\section{Proofs and Derivations}

\subsection{Factorization of the Context Posterior}
\label{appendix:factorization}
Based on the CMDP assumptions: (i) the transition kernel satisfy the Markov property conditioned on $c$, and (ii) the policy depends solely on the interaction history and not directly on the unobserved context. Consequently, the trajectory distribution given $c$ factorizes as:
\begin{align}
p(x_{1:t}|c) &= p(s_0)\pi(a_0|s_0)p(s_{1}, r_0|s_0,a_0,c)\dots \pi(a_{t-1}|x_{1:t-1})p(s_{t},r_{t-1}|s_{t-1},a_{t-1},c) \\
&= p(s_0)\prod_{i=0}^{t-1}\pi(a_i|x_{1:i})p(s_{i+1}, r_{i}|s_i, a_i, c),
\end{align}
where we assume that $s_0$ is independent of $c$.
By applying Bayes' rule and neglecting terms that are independent of $c$ (i.e., $p(s_0)$ and $\pi(a_i|x_{1:i})$), we arrive at the desired property:
\begin{align}
p(c | x_{1:t}) \propto p(c)p(x_{1:t}| c) \propto p(c)\prod_{i=0}^{t-1} p(s_{i+1},r_i| s_i,a_i,c).
\end{align}

\subsection{Probabilistic Interpretation of Sum Aggregation}
\label{appendix:permutation}
To intuitively understand why sum aggregation can be a natural choice,  consider a scenario where the context posterior $p(c | x_{1:t})$ is modeled as a product of Gaussian factors, following \citet{rakelly2019efficient}. Under the assumption of a uniform prior $p(c)$, the posterior distribution takes the form:
\begin{equation}
p(c | x_{1:t}) \propto \prod_{i=1}^t \mathcal{N}(c; \mu(x_i) , \sigma^2(x_i)) \propto \prod_{i=1}^t \exp\left(-\frac{(c-\mu(x_i))^2}{2\sigma^2(x_i)}\right) \propto \exp\left(-\frac{(c-\mu_{post})^2}{2\sigma_{post}^2}\right).
\end{equation}
Here, the posterior precision $1/\sigma_{post}^2$ and the precision-weighted mean $\mu_{post}/\sigma_{post}^2$ are determined by the sum of individual precisions and precision-weighted means:
\begin{equation}
\frac{1}{\sigma_{post}^2} = \sum_{i=1}^t \frac{1}{\sigma^2(x_i)}, \quad \frac{\mu_{post}}{\sigma_{post}^2} = \sum_{i=1}^t \frac{\mu(x_i)}{\sigma^2(x_i)}.
\end{equation}
If we define the transition embedder $E_{\psi}(x_i)$ as:
\begin{equation}
E_{\psi}(x_i) = \left(\frac{\mu(x_i)}{\sigma^2(x_i)}, \frac{1}{\sigma^2(x_i)} \right),
\end{equation}
then the accumulated memory $m_t$ becomes
\begin{equation}
m_t = \sum_{i=1}^t E_{\psi}(x_i) = \left( \frac{\mu_{post}}{\sigma_{post}^2}, \frac{1}{\sigma_{post}^2} \right).
\end{equation}
Therefore, $m_t$ exactly encodes the sufficient statistics of the posterior $p(c | x_{1:t})$.

\subsection{Proof of \cref{proposition:optimality}}
\label{appendix:prop_1}
Before presenting the main proof, we first establish a technical lemma regarding the topological properties of the multiset space. This lemma guarantees the compactness of the domain and the continuity of the summation function, which are prerequisites for establishing the homeomorphism and applying the extension theorem.

\begin{lemma}\label{lemma:lemma_1}
Let $\mathcal{X}$ be a compact subset of $\mathbb{R}^n$, and $E: \mathbb{R}^n \to \mathbb{R}^m$ be a continuous function. Then, the space of bounded multisets $\mathcal{M}_{\leq T}(\mathcal{X})$, endowed with the disjoint union topology, is compact. Furthermore, the summation function $\hat{E}: \mathcal{M}_{\leq T}(\mathcal{X}) \to \mathbb{R}^m$ defined by $\hat{E}(\{\{x_1, \dots, x_t\}\}) = \sum_{i=1}^t E(x_i)$ is continuous.
\end{lemma}
\begin{proof}
First, we establish the compactness of $\mathcal{M}_{\leq T}(\mathcal{X})$. Since $\mathcal{X}$ is compact, the product space $\mathcal{X}^t$ is compact for any $t \leq T$. The component $\mathcal{M}_t(\mathcal{X})$ is the image of $\mathcal{X}^t$ under the continuous canonical quotient map, hence it is compact. Consequently, the finite disjoint union $\mathcal{M}_{\leq T}(\mathcal{X}) = \bigsqcup_{t=0}^T \mathcal{M}_t(\mathcal{X})$ is compact.

Next, regarding the continuity of $\hat{E}$, observe that the map $(x_1, \dots, x_t) \mapsto \sum_{i=1}^t E(x_i)$ defined on $\mathcal{X}^t$ is continuous and permutation-invariant. By the definition of the quotient topology, this map induces a continuous restriction of $\hat{E}$ on each component $\mathcal{M}_t(\mathcal{X})$. Since $\hat{E}$ is continuous on each component of the disjoint union, it is continuous on the entire domain.
\end{proof}

\textbf{Proposition 3.1} \textit{
Let $E_{\psi}:\mathcal{X}\to\mathbb{R}^m$ be a one-layer MLP with output dimension $m=2nT+1$ and an analytic non-polynomial element-wise activation $f:\mathbb{R}\to\mathbb{R}$, defined as
\begin{align}
    E_{\psi}(x) = f(Ax + b), \ \text{where } \psi=(A,b).
\end{align}
Then, given a continuous optimal policy $\pi^*:\mathcal{S}\times\mathcal{M}_{\leq T}(\mathcal{X})\to\mathcal{Y}\subset\mathbb{R}^k$, there exists a continuous function $\pi:\mathcal{S}\times\mathbb{R}^m\to\mathbb{R}^k$ such that
\begin{align}
    \pi(s_t, \sum_{i=1}^t E_{\psi}(x_i)) = \pi^*(s_t, p(c|x_{1:t}))
\end{align}
for Lebesgue-almost-every $A\in\mathbb{R}^{m\times n}$ and $b\in\mathbb{R}^m$.
}

\begin{proof}
Let $\hat{E}: \mathcal{M}_{\leq T}(\mathcal{X})\to \mathbb{R}^m$ be the summation function defined as $\hat{E}(\{\{x_1,...,x_t\}\}) = \sum_{i=1}^t E_{\psi}(x_i)$. According to \citep[Theorem~3.3]{amir2023neural}, assuming $m = 2nT+1$, $\hat{E}$ is injective for Lebesgue-almost-every parameters $A$ and $b$. From \cref{lemma:lemma_1}, the domain $\mathcal{M}_{\leq T}(\mathcal{X})$ is compact and $\hat{E}$ is continuous. Since a continuous bijection from a compact space to a Hausdorff space is a homeomorphism onto its image, $\hat{E}$ is a homeomorphism between $\mathcal{M}_{\leq T}(\mathcal{X})$ and its image $\mathcal{Z} \coloneqq \hat{E}(\mathcal{M}_{\leq T}(\mathcal{X})) \subset \mathbb{R}^m$. This implies that the inverse map $\hat{E}^{-1}: \mathcal{Z} \to \mathcal{M}_{\leq T}(\mathcal{X})$ is well-defined and continuous.

Now, consider the function $g: \mathcal{S} \times \mathcal{Z} \to \mathcal{Y}$ defined by
\begin{align}
    g(s_t, z) = \pi^*(s_t, p(c|x_{1:t})), \ \text{where } \{\{x_1,...,x_t\}\} = \hat{E}^{-1}(z).
\end{align}
The continuity of $g$ follows directly from the continuity of $\pi^*$ and $\hat{E}^{-1}$. Note that $\mathcal{Z}$ is a compact, and thus closed, subset of $\mathbb{R}^m$. By the Tietze extension theorem applied component-wise to the output space $\mathbb{R}^k$ (recall $\mathcal{Y} \subset \mathbb{R}^k$), there exists a continuous extension $\pi: \mathcal{S} \times \mathbb{R}^m \to \mathbb{R}^k$ such that $\pi(s_t, z) = g(s_t, z)$ for all $z \in \mathcal{Z}$. For such $\pi$, we have
\begin{align}
    \pi(s_t, \sum_{i=1}^t E_{\psi}(x_i)) = g(s_t, \sum_{i=1}^t E_{\psi}(x_i)) =  \pi^*(s_t, p(c|x_{1:t})), 
\end{align}
which verifies the claim.
\end{proof}

\subsection{Extension of \cref{proposition:optimality} to Normalized Memory ($m=2nT+2$)}\label{appendix:extension}
To extend the guarantee of \cref{proposition:optimality} to the normalized setting, consider the original memory dimension $m_0 = 2nT+1$ and let $\tilde E_\psi:\mathcal X\to\mathbb R^{m_0}$ be the encoder defined in \cref{proposition:optimality}. We construct an augmented encoder $E_\psi$ and an offset $\Psi$ in $\mathbb{R}^{m_0+1}$ as follows:
\begin{align}
E_\psi(x) = (\tilde E_\psi(x), 0), \quad \Psi = (\mathbf{0}_{m_0}, 1).
\end{align}
Let $\tilde{m}_t = \sum_{i=1}^t \tilde E_\psi(x_i)$ denote the unnormalized memory in the original space. The memory vector before normalization then takes the form:
\begin{align}
\sum_{i=1}^t E_\psi(x_i) + \Psi  = (\tilde{m}_t, 1).
\end{align}
Consequently, the normalized memory $\hat m_t$ is obtained by projecting this vector onto the unit hypersphere:
\begin{align}
\hat m_t = \frac{(\tilde{m}_t, 1)}{\lVert (\tilde{m}_t, 1) \rVert}.
\end{align}
Observing that the normalization constant $1/{\lVert (\tilde{m}_t, 1) \rVert}$ corresponds to the last coordinate of $\hat m_t$, we can explicitly recover $\tilde{m}_t$ from $\hat{m}_t $via:
\begin{align}
\tilde{m}_t = \frac{\hat m_{t, 1:m_0}}{\hat m_{t, m_0+1}},
\end{align}
where $\hat m_{t, k}$ denotes the $k$-th component of $\hat m_t$. This continuous one-to-one correspondence between $\hat m_t$ and $\tilde{m}_t$ implies that the existence of a continuous function $\pi(s_t, \tilde{m}_t)$, established in \cref{proposition:optimality}, extends to the normalized memory formulation $\pi(s_t, \hat{m}_t)$.

\section{Experimental Details}
\label{appendix:exp_details}
This appendix provides detailed descriptions of the experimental setup that complement the main paper. We describe (i) the implementation of the memory-based RL framework, (ii) architectural design choices for memory modules and baselines, (iii) RL algorithm configurations, and (iv) environment specifications. Our implementation follows the sequence-based memory RL formulation of \cite{ni2023transformers}, with modifications tailored to our proposed memory architecture.

\subsection{Memory-Based RL Implementation}


Following the memory-based RL baseline described in \cite{ni2023transformers}, we implement a sequence-based agent that summarizes an episode into a compact memory representation and uses it as input to standard off-policy RL components. Concretely, we represent each episode as a sequence and feed the episode trajectory (observation--action history and rewards) into a sequence model to produce a latent memory vector $m_t$. This memory embedding is then consumed by the RL heads to compute the corresponding training objectives. Figure~\ref{fig:memory-baselines} illustrates the overall structure and how the learned memory connects the sequence model to the RL heads.

For training, we do not construct updates from a single transition alone; instead, we sample mini-batches consisting of \emph{entire episodes} from the replay buffer and train on episode batches. This training scheme directly follows the sequence-based baseline design in \cite{ni2023transformers}: learning a history summary from episode sequences and using it to perform off-policy RL updates. This design allows the memory module to be trained purely from the RL objective, without auxiliary sequence-prediction losses.

\begin{figure}[t]
    \centering
    \begin{subfigure}[t]{0.49\linewidth}
        \centering
        \includegraphics[width=\linewidth]{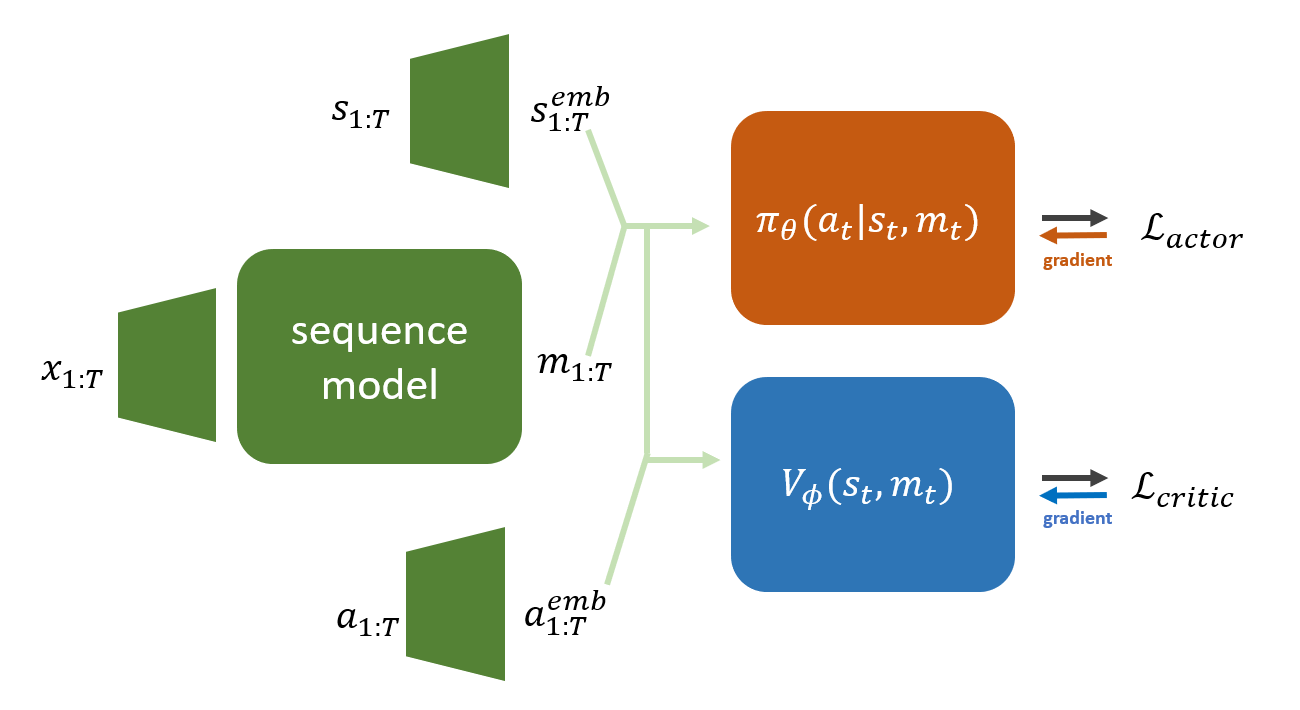}
        \caption{SAC-based memory agent for MuJoCo and Meta-World.}
        \label{fig:sac-mem}
    \end{subfigure}
    \begin{subfigure}[t]{0.49\linewidth}
        \centering
        \includegraphics[width=\linewidth]{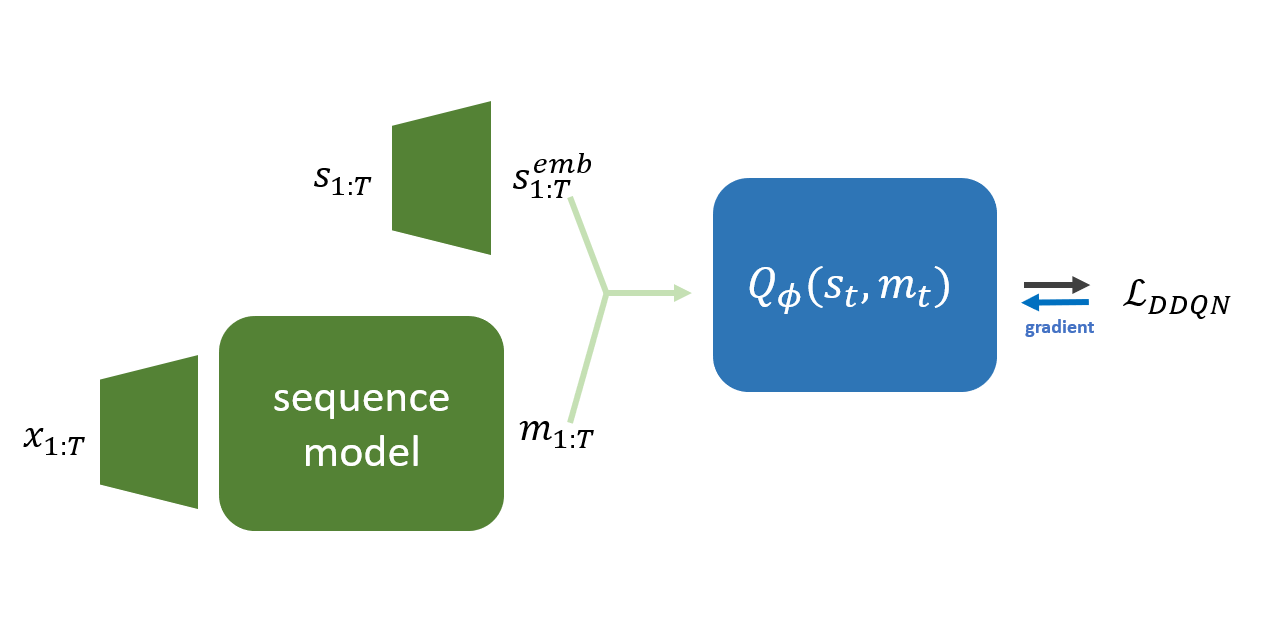}
        \caption{DDQN-based memory agent for T-Maze.}
        \label{fig:ddqn-mem}
    \end{subfigure}\hfill
    \caption{\textbf{Memory-based RL framework used in our experiments.} An episode trajectory consisting of observation, action, and reward history is processed by a sequence encoder to produce a latent memory representation $m_t$ that summarizes past interactions. This memory vector conditions standard off-policy RL components. For discrete-control tasks (e.g., T-Maze), $m_t$ is provided to a DDQN head to estimate action values. For continuous-control tasks (MuJoCo and Meta-World), $m_t$ is shared by the policy (actor) and twin Q-networks (critics) in a SAC framework. The sequence encoder, policy, and value networks are trained end-to-end, allowing historical information to influence action selection and value estimation through the learned memory representation.}
    \label{fig:memory-baselines}
\end{figure}

\subsection{Memory architecture baselines}
\label{appendix:baselines}
We compare MATE against Transformer and RNN memory baselines under closely matched capacity. Following \cite{ni2023transformers}, the Transformer baseline adopts a GPT-2 style architecture, and the RNN baseline uses an LSTM. To ensure that performance differences stem from memory structure rather than model scale, we match hidden dimensionality and depth across models. The transition encoder in MATE uses a residual feedforward design with a comparable parameter count to the Transformer blocks. Unlike GPT-style models, MATE is permutation-invariant and therefore does not rely on positional encodings. Additionally, to maintain consistency with MATE’s hyperspherical memory representation, we apply the same hyperspherical projection to state embeddings across all models when applicable.

For the experimental setup, we utilized a single layer of unit block for all baselines and MATE. The hidden dimension was set to 256 for Mujoco and Meta-World benchmarks and 128 for T-Maze experiments.

\subsection{RL algorithm Details}
For MuJoCo and Meta-World experiments, we used SAC \cite{haarnoja2018soft} to train the memory-based agent. The actor and twin critics operate on the same memory embedding $m_t=f_{\psi}(\tau_{0:t})$, and SAC is optimized with entropy regularization and target critics (using the minimum of the two critics to reduce overestimation). In particular, the actor objective can be written as
\begin{equation}
\mathcal{L}_{\pi}(\phi,\psi,\theta)
=
\mathbb{E}\Big[
\alpha \log \pi_{\phi}(a\mid m_t)
-\min_{i\in\{1,2\}} Q_{\theta_i}(m_t,a)
\Big],
\label{eq:sac_actor}
\end{equation}
where $\phi$ denotes the actor parameters and $\theta=\{\theta_1,\theta_2\}$ denotes the twin-critic parameters. When the encoder is shared, gradients of $\mathcal{L}_{\pi}$ can flow not only to $\phi$ but also to the critic and encoder parameters through the critic term $\min_i Q_{\theta_i}(m_t,a)$, which can undesirably couple the actor update with the critic (and the shared representation). To mitigate this, we enable an optional \texttt{freeze\_critic} mechanism during the actor update and instead optimize the stop-gradient variant
\begin{equation}
\mathcal{L}_{\pi}^{\text{freeze}}(\phi)
=
\mathbb{E}\Big[
\alpha \log \pi_{\phi}(a\mid \mathrm{sg}(m_t))
-\min_{i\in\{1,2\}} Q_{\bar{\theta}_i}(\mathrm{sg}(m_t),a)
\Big],
\label{eq:sac_actor_freeze}
\end{equation}
where $\mathrm{sg}(\cdot)$ denotes stop-gradient (detach) and $\bar{\theta}$ indicates that critic parameters are treated as fixed in the actor step. This enforces $\nabla_{\theta}\mathcal{L}_{\pi}^{\text{freeze}}=0$ and $\nabla_{\psi}\mathcal{L}_{\pi}^{\text{freeze}}=0$, so the critic and shared encoder are updated only via the TD critic loss, while the actor is optimized using the critic as a fixed evaluator. For a more detailed discussion and derivations of this shared-memory training issue, we refer readers to \cite{ni2023transformers}. The SAC hyperparameters for both Meta-World and Mujoco experiments are reported in Table~\ref{tab:sac_hyperparameter}.

For T-Maze experiments we adopted DDQN \cite{van2016deep} as the underlying off-policy RL algorithm, following the setup in \cite{ni2023transformers}. We employed an $\epsilon$-greedy exploration strategy with a linear decay schedule. Specifically, $\epsilon$ is initialized at 1.0 and linearly annealed to $\frac{1}{T}$ (where $T$ denotes the horizon length) over the first 10\% of total training episodes, after which it remains constant at $\frac{1}{T}$. The remaining hyperparameters are detailed in \cref{tab:ddqn_hyperparameter}.
\begin{table}[htbp]
    \centering
    \begin{minipage}{0.49\textwidth}
    \centering
    \caption{\textbf{SAC algorithm hyperparameters}}
    \begin{tabular}{cc}
    \toprule
      Hyperparameter & Value \\
    \midrule
    Network hidden size & (512, 512) \\
    Discount factor ($\gamma$) & 0.99 \\
    Target update rate ($\tau$) & 0.001 \\
    Replay buffer size & 1e4 \\
    Learning rate & 1e-4 \\
    Batch size & 64\\
    Freeze critic & true \\
    \bottomrule
    \end{tabular}
    
    \label{tab:sac_hyperparameter}
    \end{minipage}
    \hfill
    \centering
    \begin{minipage}{0.49\textwidth}
    \centering
    \caption{\textbf{DDQN algorithm hyperparameters}}
    \begin{tabular}{cc}
    \toprule
      Hyperparameter & Value \\
    \midrule
    Network hidden size & (256, 256) \\
    Discount factor ($\gamma$) & 0.99 \\
    Target update rate ($\tau$) & 0.001 \\
    Replay buffer size & 1e4 \\
    Learning rate & 3e-5 \\
    Batch size & 64 \\
    Gradient Clipping & 0.03 \\
    \bottomrule
    \end{tabular}
    
    \label{tab:ddqn_hyperparameter}
    \end{minipage}

\end{table}

\subsection{Environment Details}
\subsubsection{MuJoCo Benchmarks}
\label{appendix:mujoco}
\paragraph{Ant-Dir.}
Ant is a 3D quadruped robot with 8 torque-controlled joints. Each episode samples a latent context
$c=d^\star$ specifying a desired locomotion direction in the plane, which is not provided to the agent;
thus the optimal behavior is to infer $d^\star$ from interaction and adapt the gait to move consistently
along the target direction. We use the standard Gymnasium MuJoCo observation (proprioceptive positions,
velocities, and contact-related signals; see the Gymnasium documentation for details), with torque action
$a_t\in\mathbb{R}^8$. The reward is goal-conditioned on $d^\star$ and encourages planar velocity projected
onto $d^\star$, together with a standard control penalty. \cite{finn2017maml}

\paragraph{Cheetah-Vel.}
HalfCheetah is a planar (2D) running robot with 6 torque-controlled joints. Each episode samples a latent
context $c=v^\star$ corresponding to a target forward speed, which is not observed by the agent; the
optimal behavior is therefore to infer $v^\star$ online and regulate the gait to match the target speed.
We use the standard Gymnasium MuJoCo observation (joint positions and velocities; see the Gymnasium
documentation for details), with torque action $a_t\in\mathbb{R}^6$. The reward is conditioned on $v^\star$
and encourages tracking the achieved forward velocity to the target, with a standard control penalty.
\cite{finn2017maml}

\paragraph{Walker-Param.}
Walker2d is a planar (2D) biped robot with 6 torque-controlled joints. Each episode samples a latent
context $c$ that parameterizes the environment dynamics (e.g., masses/inertias, joint damping, and
friction/contact coefficients) and is not revealed to the agent; the optimal behavior is to identify the
current dynamics from interaction history and adapt control accordingly (rather than executing a single
fixed gait). We use the standard Gymnasium MuJoCo observation (proprioceptive positions and velocities;
see the Gymnasium documentation for details), with torque action $a_t\in\mathbb{R}^6$. The reward is the
standard forward-locomotion objective with a control penalty and is fixed across contexts.

\paragraph{Hopper-Param.}
Hopper is a planar (2D) one-legged robot with 3 torque-controlled joints. Each episode samples a latent
context $c$ that changes the environment dynamics through randomized physical parameters (e.g., mass/inertia
scalings, damping, and friction/contact coefficients) and is not observed by the agent; the optimal behavior
is to infer the episode's dynamics and adapt balance and thrust timing to maintain stable hopping under the
current dynamics. We use the standard Gymnasium MuJoCo observation (proprioceptive positions and velocities;
see the Gymnasium documentation for details), with torque action $a_t\in\mathbb{R}^3$. The reward is the
standard forward-hopping objective with a control penalty and is fixed across contexts.

\subsubsection{Meta-World}
\label{appendix:metaworld}
\paragraph{Meta-World ML10.}
We additionally evaluate on Meta-World ML10 \cite{yu2020meta}. Unlike the MuJoCo CMDP tasks and T-Maze described above, this setting does not correspond to a fixed task per environment. Instead, ML10 defines a single manipulation environment distribution in which both the task type and task-specific configurations vary across episodes. At the beginning of each episode, one of 10 manipulation tasks (\emph{reach, push, pick-place, sweep-into, basketball, drawer-close, window-open, button-press, dial-turn,} and \emph{peg-insert-side}) is sampled, and the task identity is not provided to the agent. In addition to the change in task type, object initial positions, goal locations, and relevant target configurations are randomized within each task, so both high-level objectives and low-level scene layouts vary across episodes. All tasks share the same embodiment: a simulated 7-DoF Sawyer arm with a parallel gripper operating in a common tabletop workspace. We use the standard Meta-World low-dimensional state observations, including robot proprioceptive features, end-effector pose, and task-relevant object and goal information, and a continuous 4-dimensional action space corresponding to 3D end-effector displacement control and a gripper open/close command. We use a fixed episode horizon of $H=300$ steps for all ML10 experiments. This formulation requires the agent to infer the underlying task and adapt its behavior from interaction history within a unified environment family rather than across separate task-specific environments.

\subsubsection{T-Maze}
\label{appendix:tmaze_setting}
We adopt and modify the T-Maze environment originally introduced by \cite{ni2023transformers}. The environment is situated in a discrete 2-dimensional space containing a T-shape structure. It consists of a 1-dimensional corridor of length $L$, extending from an oracle state $O$ to a junction state $J$. At the junction, the path branches into two directions (Up and Down), leading to two potential goal states, $G_{1}$ and $G_{2}$.

At the beginning of each episode, the true goal $G$ is randomly initialized as either $G_{1}$ or $G_{2}$. Crucially, the context information regarding the true goal is observable only at the oracle state $O$. The agent is initialized at a start position $S$ within the corridor. To solve the task, the agent must traverse the corridor, identify the correct goal by visiting $O$ (if necessary), and proceed to $J$ to turn toward the target. The action space consists of four discrete movements: Left, Right Up, and Down. 

\paragraph{Adaptation to CMDP Framework.} The original T-Maze is designed as a POMDP, where the agent receives discrete indicator observation about current state. To align this environment with the CMDP framework, where the state is typically fully observable while the context is hidden, we introduce two key modifications. First, regarding state representation, we define the state as the 2-dimensional coordinates of the agent instead of discrete indicators. Second, regarding reward shaping, we replace the original time-dependent penalty with a velocity-based penalty to provide a time-independent learning signal. The total reward consists of a sparse goal reward $R_{goal} = \mathbb{1}(s_{t+1}=G)$ and penalty $R_{P}$. The penalty is formulated as:
\begin{align*}
    R_{p}(h_{1:t}, a_{t}) = \frac{1- (x_{t}-x_{t-1})}{L}
\end{align*}
where $x_{t}$ denotes the agent's position along the corridor at time $t$. This term drives the agent towards the junction.

\paragraph{Passive.} In the \textit{Passive} setting, the start position coincides with the oracle position ($S=O$), and the corridor length is $L=T-1$. Since the agent observes the goal cue immediately at the start, no exploration is required. The optimal policy simply moves right for $T-1$ steps and move toward the goal $G$. The expected return for the optimal policy is 1.0, while an optimal Markovian policy yields 0.5. The worst-case return is -1.0.

\paragraph{Active.} In the \textit{Active} setting, the oracle is located one step to the left of the start position, and the corridor length is $L=T-2$. This necessitates active information gathering: the optimal policy must move left to $O$, observe the context, and then proceed right. To test active exploration, we introduce two constraints: distractor noise, where the agent receives a false goal cue at the start state, and physical constraint, where movement towards the junction is blocked at the first time step. These modifications force the agent to actively seek the oracle rather than relying on initial observations. Despite these challenges, the expected returns for optimal, Markovian policies remain consistent with the passive variant.

\end{document}